\documentclass{article}

\usepackage[square,numbers]{natbib}
\usepackage[preprint]{neurips_2024}

\usepackage[utf8]{inputenc} 
\usepackage[T1]{fontenc}    
\usepackage{hyperref}       
\usepackage{url}            
\usepackage{booktabs}       
\usepackage{amsfonts}       
\usepackage{nicefrac}       
\usepackage{microtype}      
\usepackage{xcolor}         

\usepackage{mathrsfs}
\usepackage{microtype}
\usepackage{graphicx}
\usepackage{subfigure}
\usepackage{booktabs} 
\usepackage{amsmath}
\usepackage{amssymb}
\usepackage{mathtools}
\usepackage{amsthm}
\usepackage{enumitem}
\usepackage[capitalize,noabbrev]{cleveref}

\theoremstyle{plain}
\newtheorem{theorem}{Theorem}[section]

\theoremstyle{definition}
\newtheorem{definition}[theorem]{Definition}

\theoremstyle{remark}

\usepackage[textsize=tiny]{todonotes}

\newcommand{\ourmethod}[0]{\textsc{EAP}}
\usepackage{dsfont}

\newcommand{\numberset}{\mathbb} %

\newcommand{\R}{\numberset{R}}

\newcommand{\E}{\numberset{E}}
\newcommand{\1}{\mathds{1}}
\newcommand{\Pp}{\mathbb{P}}

\newcommand{\xreal}{x_{\textsc{r}}}
\newcommand{\xind}{x_{\textsc{i}}}
\newcommand{\xunr}{x_{\textsc{u}}}

\usepackage{scalerel,stackengine} 
\stackMath
\newcommand\reallywidehat[1]{%
\savestack{\tmpbox}{\stretchto{%
  \scaleto{%
    \scalerel*[\widthof{\ensuremath{#1}}]{\kern-.6pt\bigwedge\kern-.6pt}%
    {\rule[-\textheight/2]{1ex}{\textheight}}
  }{\textheight}%
}{0.5ex}}%
\stackon[1pt]{#1}{\tmpbox}%
}

\DeclareMathOperator*{\argmin}{arg\,min}

\usepackage{tcolorbox}

\title{Uncertainty-aware Evaluation of Auxiliary Anomalies with the Expected Anomaly Posterior}

\author{%
  Lorenzo Perini\footnote{} \\
  DTAI lab \& Leuven.AI, \\
  KU Leuven, Belgium \\
  \And
  Maja Rudolph \\
  Bosch Center for AI, USA \\
  University of Wisconsin-Madison, USA\\
  \AND
  Sabrina Schmedding \\
  Bosch Center for AI, Germany \\
  \And
  Chen Qiu \\
  Bosch Center for AI, USA \\
}

\begin{document}

\maketitle

\def\thefootnote{*}\footnotetext{Work done during the internship at Bosch Center for AI. Correspondence: \href{mailto:lorenzo.perini@kuleuven.be}{\texttt{lorenzo.perini@kuleuven.be}}}
\def\thefootnote{\arabic{footnote}}

\begin{abstract}
  Anomaly detection is the task of identifying examples that do not behave as expected. Because anomalies are rare and unexpected events, collecting real anomalous examples is often challenging in several applications. In addition, learning an anomaly detector with limited (or no) anomalies often yields poor prediction performance. One option is to employ auxiliary synthetic anomalies to improve the model training. However, synthetic anomalies may be of poor quality: anomalies that are unrealistic or indistinguishable from normal samples may deteriorate the detector's performance. Unfortunately, no existing methods quantify the quality of auxiliary anomalies. We fill in this gap and propose the expected anomaly posterior (\ourmethod{}), an uncertainty-based score function that measures the quality of auxiliary anomalies by quantifying the total uncertainty of an anomaly detector. Experimentally on $40$ benchmark datasets of images and tabular data, we show that \ourmethod{} outperforms $12$ adapted data quality estimators in the majority of cases.
\end{abstract}

\section{Introduction}

Anomaly detection aims at identifying the examples that do not conform to the normal behaviour~\cite{chandola2009anomaly}. Anomalies are often connected to adverse events, such as defects in production lines~\cite{wang2024improving}, excess water usage~\cite{perini2023estimating}, failures in petroleum extraction~\cite{marti2015anomaly}, or breakdowns in wind turbines~\cite{perini2022transferring}. Detecting anomalies in time can reduce monetary costs and protect resources from harm.
For this reason, there has been significant effort to develop data-driven methods for anomaly detection.



Unfortunately, anomalies are inherently rare and sparse, which makes collecting them hard. 
As a result, the data used to train data-driven methods for anomaly detection only contains a limited number of anomalies.  
In autonomous driving, for instance, sensor malfunctions or unexpected pedestrian movements are infrequent but critical to address. As an additional challenge, available anomalies rarely represent all potential cases due to the unpredictable nature of these events. In financial fraud detection, new techniques and schemes are constantly emerging, meaning previously identified anomalies do not cover future fraudulent methods. These factors highlight the difficulty of obtaining comprehensive anomaly datasets, as new and unpredictable anomaly types are inherent to the very nature of these events.


With recent improvements in generative modeling (e.g. ~\cite{ho2020denoising,dhariwal2021diffusion}) it seems natural to introduce auxiliary anomalies, e.g. for training an anomaly detector \citep{murase2022algan}, or for model selection ~\citep{fung2023model}.  However, there are several failure cases for generated anomalies that should not be neglected. For one, auxiliary anomalies might be too similar to normal examples. For instance, the defects introduced into images of products might be imperceptible, making the image indistinguishable from a normal counterpart. On the other hand, the quality of a generated anomaly also deteriorates as it becomes too unrealistic, e.g. when the generated defect of a product is too severe. Including poor-quality anomalies for training a model is likely to harm its performance~\cite{hendrycks2019oe,qiu2022latent,li2023deep,li2023zero}. Although \citet{chen2021atom,ming2022poem} proposed sampling methods for selecting informative anomalies during training, there is no approach for quantifying the quality of auxiliary anomalies.

In this paper, we close this critical gap by introducing the \textbf{Expected Anomaly Posterior (\ourmethod{})}, the first example-wise score function that measures the quality of auxiliary anomalies. Our approach relies on a fundamental insight: high-quality auxiliary anomalies must fulfill two criteria — they must be (1) distinguishable from normal examples in the training data and (2) realistic, i.e. similar to the training examples (e.g., scratches only affect few pixels, leaving an anomalous image relatively similar to a normal one). Finding a balance between these two characteristics poses an inherent challenge. On one hand, auxiliary anomalies risk deteriorating an anomaly detector's performance if they closely resemble normal examples. On the other hand, they become less useful the more dissimilar from the training data.

Building upon this insight, we adopt a Bayesian framework to model the uncertainty of a detector's prediction. This framework accounts for both an example's dissimilarity from the normal class (via class-conditional probability) and its realism (via example density). The expectation of the posterior probability that an example is anomalous reflects our concept of the quality of an auxiliary anomaly: the approximation we derive in \Cref{sec:method} to compute the \ourmethod{} will give lower scores to indistinguishable and unrealistic anomalies.



In summary, we make three following contributions. 
\begin{itemize}[nolistsep,leftmargin=*]
\item In \Cref{sec:method}, we compute the expected anomaly posterior (\ourmethod{}), which measures the quality of an anomaly by accounting both for aleatoric and epistemic uncertainty.
\item In \Cref{sec:theory}, we provide a theoretical analysis of \ourmethod{}, including its properties and guarantees. 
\item In \Cref{sec:exp}, we run an extensive experimental analysis and show that \ourmethod{} enhances a detector's performance when using high-quality anomalies to enrich training or perform model selection.
\end{itemize}


\section{Related Work}
\paragraph{Anomaly detection.} Designing an anomaly detector requires developing a way to assign real-valued anomaly scores to the examples~\cite{chandola2009anomaly,han2022adbench}, where the higher the score, the more anomalous the example. Existing approaches often rely on heuristic intuitions about expected anomalous behavior~\cite{pang2021deep,qiu2023self}. Propagation-based detectors, such as those using proximity to training examples, assume similar instances share the same label (e.g., \textsc{SSDO})~\cite{vercruyssen2018semi}. Loss-based detectors, on the other hand, learn a decision boundary (e.g., a hypersphere over normals) and assign scores based on the distance to this boundary~\cite{ruff2019deep,zhou2021feature,gao2021connet,Qiu2022RaisingTB}. Self-supervised detectors learn models through solving auxiliary tasks and score anomalies according to model performance on self-supervised tasks~\cite{golan2018deep,bergman2020classification,qiu2021neural}. Recently, foundation models have enabled zero-shot anomaly detection~\cite{jeong2023winclip}, which overcomes the need to collect anomalies for training but still requires their use in model selection~\cite{fung2023model}.

\paragraph{Data quality.} Traditional quality score functions evaluate training examples by (1) defining a utility function that takes as input a subset of the training set and measures the performance of the model, and (2) finding a function that assigns a score to an example by quantifying its impact on the model's performance when included/excluded for training~\cite{yoon2020data,jiang2023opendataval}. Methods like leave-one-out (\textsc{Loo}) iteratively remove one example at a time to observe how test performance varies. Various techniques, such as \textsc{DataShap}~\cite{ghorbani2019data}, \textsc{BetaShap}~\cite{kwon2021beta}, \textsc{KNNShap}~\cite{jia2019efficient}, \textsc{DataBanzhaf}~\cite{wang2023data}, and \textsc{AME}~\cite{lin2022measuring}, compute the marginal contribution of an example by bootstrapping the training set and assessing its impact on model training. \textsc{DataOob} is an out-of-bag (OOB) evaluator that measures out-of-bag accuracy variation. Other methods like \textsc{Lava} and influence functions (\textsc{Inf}) quantify how the utility changes when a specific example is more weighted. The Supplement~\ref{sec:data_quality_estimators} provides a detailed overview. Unfortunately, all existing data quality estimators focus on evaluating the impact of training examples and are not tailored to estimate the quality of an external anomaly.

\section{Methodology}\label{sec:method}
In this Section, we introduce the problem setup and notations (\Cref{subsec:setup_and_notation}), and describe our proposed approach for quantifying the quality of auxiliary anomalies (\Cref{subsec:defake}).

\subsection{Problem setup}\label{subsec:setup_and_notation}
Let $(\Omega, \mathcal{F}, \Pp)$ be a probability space, and $X \colon \Omega \to \R^d$, $Y \colon \Omega \to \{0,1\}$ two random variables representing, respectively, feature vectors and class labels ($0$ for normals, $1$ for anomalies). A training dataset is an i.i.d. sample of pairs $D = \{(x_1,y_1), \dots, (x_n,y_n)\} \sim \Pp(X,Y)$ drawn from the joint distribution. Because of the rarity of anomalies, we assume to have only $m << n$ (labeled) examples from the anomaly class, in addition to $n-m$ (labeled) normal examples.

Since anomalies provide valuable training signals, but are so rare to acquire, \citet{hendrycks2019oe,murase2022algan} propose using auxiliary anomalies during training. With significant improvements in generative modeling, there are many candidate methods for generating synthetic anomalies to complement the training data. Our goal is to evaluate candidate synthetic anomalies with a quality score function $\phi$ such that higher scores indicate that the synthetic anomaly is useful for training a detector. Before formalizing our research task, we introduce the following definition.
\begin{definition}[Categorization of Anomalies]\label{def:categorizationanomalies}
Given the examples $\xreal{}, \xunr{}, \xind{} \in \R^d$, we define that 
\begin{itemize}[nolistsep,noitemsep,leftmargin=*]
\item $\xreal{}$ is a \emph{realistic anomaly} if it has high conditional probability and non-zero density
\begin{equation*}
    \Pp(Y=1|X=\xreal{}) \in [0.5, 1] \ \text{ and } \ \Pp(X=\xreal{}) > 0;
\end{equation*}
\item $\xunr{}$ is an \emph{unrealistic anomaly} if it has high conditional probability and null density 
\begin{equation*}
    \Pp(Y=1|X=\xunr{}) \in [0.5, 1] \ \text{ and } \ \Pp(X=\xunr{}) = 0 ;
\end{equation*}
\item $\xind{}$ is an \emph{indistinguishable anomaly} if it has null conditional probability and non-zero density
\begin{equation*}
    \Pp(Y=1|X=\xind{}) = 0 \ \text{ and } \ \Pp(X=\xind{}) > 0.
\end{equation*}
\end{itemize}
\end{definition}

An anomaly detector is a function $f \colon \R^d \times \{0,1\} \to \R$ that assigns a real-valued anomaly score $f(x)$ to any $x \in \R^d$. The detector $f$ is learned using the training set $D$, and can be used for estimating the conditional probability $\Pp(Y|X)$ by mapping the scores to $[0,1]$~\cite{kriegel2011interpreting}.


\begin{description}
    \item[Given:] $D$ with $m \ll n$ anomalies, a set of $l$ auxiliary anomalies $\{x \in \R^d \}$ , and a detector $f$;
    \item[Challenge:] Design a quality score function $\phi : \R^d \to \R$ for the auxiliary anomalies, such that for any realistic anomaly $\xreal \in l$ and any unrealistic or indistinguishable anomaly $\xunr{},\xind{} \in l$, the realistic anomaly has a higher quality score $\phi(\xreal)> \phi(\xunr{}), \phi(\xind{})$.
\end{description}

With this categorization, estimators for $\Pp(Y|X)$ alone cannot differentiate between $\xreal{}$ and $\xunr{}$, while estimators for $\Pp(X)$ alone cannot differentiate between $\xreal{}$ and $\xind{}$, thus not qualifying as good quality estimators.
Intuitively, a score must quantify the conditional probability to distinguish between $\xreal{}$ and $\xind{}$ but the estimate needs to account for an example's density. Roughly speaking, the lower an example's density the more uncertain the estimate, because the lack of similar training data prevents a model from learning the correct probability. This introduces an additional level of uncertainty (namely, epistemic), which requires a Bayesian perspective to be properly measured~\cite{hullermeier2021aleatoric,bengs2022pitfalls}.

\subsection{The Expected Anomaly Posterior}\label{subsec:defake}




Capturing a detector's uncertainty is challenging because one needs to account for (1) the example's proximity to the normal class (i.e., the aleatoric uncertainty) and (2) the lack of training data in the region where the example falls (i.e., the epistemic uncertainty). This is particularly complicated in anomaly detection because the epistemic uncertainty tends to be high for most anomalies, as they often fall in low-density regions~\cite{bengs2023second}.

We propose \ourmethod{}, a novel approach that estimates the quality of auxiliary anomalies by capturing an anomaly detector's uncertainty. The key idea is to model each example's probability of being an anomaly $p_x$. The quality score we propose is the expected posterior of this parameter.

\paragraph{Assumption.} For any $x$ the class conditional distribution $Y=1|X=x$ is a Bernoulli
\begin{equation*}
\Pp(Y|X=x) = \textsc{Bernoulli}(p_{x}), \quad p_{x} \sim \textsc{Beta}(\alpha_0,\beta_0).
\end{equation*}
The parameter $p_{x}$ can be interpreted as the probability of the example $x$ being an anomaly. With the Beta prior we can incorporate prior knowledge such as the expected ratio of anomalies in the data~\cite{perini2023estimating}.
Since we have at most one observation of $Y$ for each $x$, we follow \citet{charpentier2020posterior} and model the posterior over $p_x$ by conditioning on pseudo observations. 
Given $N$ pseudo observations $\bar{y}_1, \dots, \bar{y}_N$ hypothetically drawn from $\Pp(Y|X=x)$, the posterior,
\begin{equation}\label{eq:posterior_uncertainty}
    p_{x}| \bar{y}_1, \dots, \bar{y}_N \sim \textsc{Beta}(\alpha_0 + \alpha_1, \beta_0 + N - \alpha_1),
\end{equation}
is conjugate to the Beta prior, where $\alpha_1 = \sum_{i=1}^N \1 (\bar{y}_i=1)$ is the number of anomalies in the pseudo observations. That is, if we could sample $N$ labels for the same example $x$, i.e. $(x, y_1),\dots, (x, y_N)$, we would derive the posterior distribution of $Y|X=x$ by using a simple Bayes update (Eq.~\ref{eq:posterior_uncertainty}). However, sampling $N$ labels for the same example $x$ is practically impossible. Thus, we need to parametrize $\alpha_1$. Roughly speaking, if we drew $n$ training examples from $\Pp(X,Y)$ we would expect to draw $N = n \cdot \Pp(X=x)$ examples with features $x$, among which $\Pp(Y=1|X=x)$ are anomalies:
\begin{equation}
    \alpha_1 \approx n \times \reallywidehat{\Pp(Y\!=\!1,X\!=\!x)} \approx n \times \underbrace{\reallywidehat{\Pp(Y\!=\!1|X\!=\!x)}}_{\text{conditional probability}} \times \underbrace{\reallywidehat{\Pp(X\!=\!x)}}_{\text{data density}} 
\end{equation}
where $\hat{\cdot}$ indicates that the quantity is estimated. We describe how we compute both terms below. 
The expectation of the posterior in \Cref{eq:posterior_uncertainty} reflects the quality of an auxiliary anomaly $x$: if the expected posterior is high, the evidence is enough to rely on the expected conditional probability for evaluating the auxiliary anomaly, while if it is low, the quality reflects our prior belief.

\paragraph{Estimating the data density.}
Computing $\reallywidehat{\Pp(X\!=\!x)}$ has two main challenges. First, most traditional density estimators suffer the well-known curse of dimensionality~\cite{verleysen2005curse,bengio2005curse}. Second, deep estimators (e.g., Normalizing Flows~\cite{kobyzev2020normalizing}) are prohibitively time-consuming to be employed for data quality scores. Thus, \ourmethod{} relies on the rarity score~\cite{han2022rarity}, which is fast to compute and weakly affected by the curse of dimensionality. 
The rarity score (1) creates k-NN spheres centered on each training example, and (2) assigns the smallest radius of the sphere that contains the given synthetic example. If the synthetic example falls outside of all spheres, it is considered too uncommon and gets null rarity. 
We use the rarity score $r_{\hat{k}}$ with an estimated $\hat{k}$ to estimate the data density.\footnote{We explain how we compute $\hat{k}$ in \Cref{sec:rarity_score} in the Supplement.}
Intuitively, the density behaves as the inverse of the rarity score: highly uncommon examples should have low density. Thus, we take the reciprocal value of the rarity score and normalize it using the training rarity scores:
\begin{equation}\label{eq:density}
    \reallywidehat{\Pp(\!X\!=\!x\!)} = \frac{1 \slash r_{\hat{k}}(x)}{1 \slash r_{\hat{k}}(x)+\sum_{i=1}^n 1 \slash r_{\hat{k}}(x_i)}.
\end{equation}

\paragraph{Estimating the conditional probability.}
Computing $\reallywidehat{\Pp(Y\!=\!1|X\!=\!x)}$ in anomaly detection is a hard task because (1) class probabilities are generally unreliable for imbalanced classification tasks~\cite{wallace2012class,tian2020posterior}, and (2) the available anomalies might be non-representative of the whole anomaly class (i.e., we have access to a biased set). This makes traditional calibration techniques often impractical~\cite{silva2023classifier,deng2022cadet}. However, we mainly care about having probabilities that satisfy two properties. First, they must be coherent with the detector's prediction, namely a predicted anomaly (normal) needs a probability greater (lower) than $0.5$. Second, we want the proportion of predicted anomalies to match the expected proportion of true anomalies. This guarantees that, if the detector's ranking is accurate, the class predictions are optimally computed.

For this task, we employ a squashing scaler~\cite{vercruyssen2018semi} to map the anomaly scores to $[0,1]$ probability values
\begin{equation}\label{eq:classconditionaprob}
\reallywidehat{\Pp(\!Y\!=\!1|X\!=\!x\!)} = 1-2^{-\left(\frac{f(x)}{\lambda}\right)^2}
\end{equation}
where $\lambda$ is the detector's decision threshold which we set such that the number of training examples with probability $> 0.5$ (after mapping) is equal to the number of training anomalies $m$~\cite{perini2023estimating}.

\paragraph{Computing the quality scores.} Using our point estimates for the data density and the conditional probability, the parametrized posterior distribution $p_{x}| \bar{y}_1, \dots, \bar{y}_N$ can be computed by substituting $\alpha_1 = n \reallywidehat{\Pp(Y\!=\!1|X\!=\!x)} \reallywidehat{\Pp(X\!=\!x)}$ and $N = n \reallywidehat{\Pp(X\!=\!x)}$ to Eq.~\ref{eq:posterior_uncertainty}. Finally, we compute the quality of an auxiliary anomaly $x$ by taking its expectation
\begin{align*}
\phi(x) = \E[p_{x}| \bar{y}_1, \dots, \bar{y}_N] = \frac{\alpha_0 + n \reallywidehat{\Pp(X=x)} \reallywidehat{\Pp(Y=1|X=x)}}{\alpha_0 + \beta_0 + n \reallywidehat{\Pp(X=x)}}.
\end{align*}

\section{Theoretical Analysis of \ourmethod{}}
\label{sec:theory}
We theoretically investigate two tasks. First, we illustrate the main properties of \ourmethod{}, namely how it behaves when subject to (1) large training sets ($n \to +\infty)$, (2) small training sets or zero-density examples, (3) high-class conditional probabilities.
Second, we answer the following question: \emph{Given a realistic anomaly $\xreal{}$, an indistinguishable anomaly $\xind{}$ and an unrealistic anomaly $\xunr{}$ as in \Cref{def:categorizationanomalies}, does \ourmethod{} rank $\phi(\xreal{})>\phi(\xind{}),\phi(\xunr{})$?
}

\subsubsection*{i) \bf{\ourmethod{} has three relevant properties:}}

\emph{P1) Convergence to class conditional probabilities.} The number of training examples indicates how strong the empirical evidence is. That is, the detector $f$ has enough evidence to estimate properly the class conditional probability. Thus, for high $n$, \ourmethod{} converges to the class conditional probability
\begin{equation*}
    \phi(x) \to \reallywidehat{\Pp(Y\!=\!1|X\!=\!x)} \qquad \text{for } n \to +\infty;
\end{equation*}

\emph{P2) Convergence to the prior's mean.} No empirical evidence implies that the posterior remains equal to the prior. Thus, \ourmethod{} assigns the prior's mean for relatively small $n$ or a null-density region, 
\begin{equation*}
    \phi(x) \to \frac{\alpha_0}{\alpha_0+\beta_0} \qquad \text{for } n \to 0 \ \text{ or } \ \reallywidehat{\Pp(X=x)} \to 0.
\end{equation*}

\emph{P3) The quality of distinguishable anomalies increases with their density.} Given $\reallywidehat{\Pp(Y=1|X=x)} \approx 1$ for an example $x$, its quality depends only on its density: the closer/more similar to the training examples, the higher the density, the higher its quality:
\begin{equation*}
    \phi(x) \approx 1 - \frac{\beta_0}{\alpha_0+\beta_0 + n\Pp(X=x)}.
\end{equation*}

\subsubsection*{ii) \bf{\ourmethod{}' ranking guarantee.}}
We show that \ourmethod{} ranks the anomalies as (1st) realistic, (2nd) unrealistic, and (3rd) indistinguishable.

\begin{theorem}
Let $\xreal{}, \xunr{}, \xind{} \in \R^d$ be, respectively, a realistic, unrealistic, and indistinguishable anomaly. If the estimators in Eq.~\ref{eq:density} and Eq.~\ref{eq:classconditionaprob} satisfy the properties of \Cref{def:categorizationanomalies}, then
\begin{equation}
\frac{\alpha_0}{\alpha_0+\beta_0}<0.5 \implies \phi(\xreal{})>\phi(\xunr{}) > \phi(\xind{}).
\end{equation}
\end{theorem}
\begin{proof}
Using the definition of indistinguishable and unrealistic anomaly, we immediately conclude
\begin{equation*}
    \phi(\xunr{}) = \frac{\alpha_0}{\alpha_0+\beta_0} >  \frac{\alpha_0}{\alpha_0+\beta_0 + n\Pp(X=\xind{})} = \phi(\xind{})
\end{equation*}
because $\Pp(X=\xind{})>0$. As a second step, we assume that $\frac{\alpha_0}{\alpha_0+\beta_0}<0.5$ and show algebraically that
\begin{equation*}
\begin{split}
&\phi(\xreal{})>\phi(\xunr{}) \!\iff \!\phi(\xreal{})-\phi(\xunr{})>0  \!\iff \!\frac{\alpha_0 + n \Pp(X\!=\!\xreal{}) \Pp(Y\!=\!1|X\!=\!\xreal{})}{\alpha_0 + \beta_0 + n \Pp(X\!=\!\xreal{})}\!-\!\frac{\alpha_0}{\alpha_0+\beta_0}\!>\!0  \\
&\iff n\Pp(X\!=\!\xreal{})\left[(\alpha_0+\beta_0)\Pp(Y\!=\!1|X\!=\!\xreal{}) \!-\! \alpha_0\right]\!>0 \iff \Pp(Y\!=\!1|X\!=\!\xreal{}) > \frac{\alpha_0}{\alpha_0+\beta_0},
\end{split}
\end{equation*}
which holds as $\Pp(Y\!=\!1|X\!=\!\xreal{})>0.5> \frac{\alpha_0}{\alpha_0+\beta_0}$.
\end{proof}

\section{Experiments}\label{sec:exp}

We empirically evaluate three aspects of our method \ourmethod{}: (a) whether it measures properly the quality of auxiliary anomalies, and (b) its impact on selecting auxiliary anomalies for learning a model or (c) for model selection. To this end, we address the following five experimental questions:

\begin{itemize}[nolistsep]
    \item[Q1.] How does \ourmethod{} compare to existing methods at assigning quality scores?
    \item[Q2.] How does a model’s performance vary when including \emph{high-quality} anomalies for training?
    \item[Q3.] How does a model’s performance vary when including \emph{low-quality} anomalies for training?
    \item[Q4.] How does the performance of a CLIP-based zero-shot anomaly detection method vary when using the selected auxiliary anomalies for prompt tuning?
    \item[Q5.] How do \ourmethod{}' scores vary for different priors?
\end{itemize}

\subsection{Experimental Setup}

\paragraph{Baselines.}
We compare \ourmethod{}\footnote{Code is available at: URL provided upon acceptance.} with $12$ adapted baselines:
\textsc{Loo}, \textsc{kNNShap}~\cite{jia2019efficient}, \textsc{DataBanzhaf}~\cite{wang2023data}, \textsc{AME}~\cite{lin2022measuring}, \textsc{LavaEv}~\cite{just2023lava}, \textsc{Inf}~\cite{feldman2020neural}, and \textsc{DataOob}~\cite{kwon2023data} are existing data quality evaluators that measure the impact of a training example on the model performance. We adapt these methods by including each auxiliary anomaly (individually) in the training set and evaluating its contribution. \textsc{RandomEv} assigns uniform random scores to each auxiliary anomaly. \textsc{Rarity}~\cite{han2022rarity} computes the rarity score of each auxiliary anomaly. Finally, we include the estimators for the data density $P_x$, the class conditional probability $P_{y|x}$, and a linear combination of them $P_{y|x}+ NP_x$.

\paragraph{Data.}

We carry out our study on $40$ datasets, including $15$ widely used benchmark image datasets (\textsc{MvTec})~\cite{bergmann2019mvtec}, $3$ industrial image datasets for Surface Defect Inspection (SDI)~\cite{Wang_2022_BMVC}, and an additional $22$ benchmark tabular datasets for anomaly detection with semantically useful anomalies, commonly referenced in the literature~\cite{han2022adbench}. These datasets vary in size, feature count, and anomaly proportion.

For each dataset, we construct an auxiliary set of $l$ anomalies by combining realistic, indistinguishable, and unrealistic anomalies ($\frac{l}{3}$ each). Realistic anomalies are labeled anomalies provided with the dataset, indistinguishable anomalies are labeled normal examples with swapped labels, and unrealistic anomalies come from other datasets. Specifically, to collect unrealistic anomalies we randomly select $5$ datasets out of the $40$, subsample them to a specific example count, and fix their dimensionality to $d$ using random projections (either extending or reducing it). \emph{Pseudo-quality} labels ``good'' and ``poor'' are assigned to real anomalies and the other two groups, respectively, reflecting the ground truth where real anomalies should have high-quality scores.



\paragraph{Setup.}
For each dataset, we proceed as follows: 
(i) We create a balanced test set by adding random normal examples and $50\%$ of available anomalies; 
(ii) We generate a set of $l$ auxiliary anomalies as described above with $\frac{l}{3} = 40\%$ of available anomalies;
(iii) We create a training set by adding $10\%$ of available anomalies and all remaining normal examples to the training set.
(iv) We apply all methods to evaluate the external set of anomalies, using the training set for validation when required (as $m \ll n$, we avoid partitioning the training set); 
To mitigate noise, steps (i)-(iv) are repeated $10$ times with different seeds, resulting in a total of $4000$ experiments (datasets, methods, seeds). While computing \ourmethod{} is fast, the baselines have high computational costs because they train a model several times. To run all experiments, we use an internal cluster of six 24- or 32-thread machines (128 GB of memory). The experiments finish in $\sim 72$ hours.


\paragraph{Models and Hyperparameters.}
For all baselines, we use \textsc{SSDO}~\cite{vercruyssen2018semi} as the underlying anomaly detector $f$ with $k=10$ and Isolation Forest~\cite{liu2008isolation} as prior. When exposed to selected auxiliary anomalies, we employ an \textsc{SVM} with RBF kernel (for images) and a \textsc{Random Forest} (for tabular data) to make the normal vs. abnormal classification. For images, we use the pre-trained \textsc{ViT-B-16-SigLIP}~\cite{zhai2023sigmoid} to extract the features from images and use them as inputs to \ourmethod{} and all baselines.
Our method \ourmethod{} has one hyperparameter, namely the prior $\alpha_0, \beta_0$, which we set to $\frac{m}{n}$ (the proportion of anomalies in the training set) and $1-\frac{m}{n}$. Intuitively, this corresponds to the expected proportion of (real) anomalies if an external dataset was sampled from $\Pp(X,Y)$. The baselines\footnote{Code: \url{https://github.com/opendataval}} have the following hyperparameters: \textsc{kNNShap} and \textsc{Rarity} have $k=10$, \textsc{DataBanzhaf}, \textsc{AME}, \textsc{Inf} and \textsc{DataOob} use $50$ models. All other hyperparameters are set as default~\cite{soenen2021effect}.

\paragraph{Evaluation Metrics.}
We employ \textbf{four} evaluation metrics. First, we use the Area Under the Receiving Operator Curve (\textbf{AUC$_{\textsc{qlt}}$}) to evaluate the methods' ability to rank good-quality examples higher than poor-quality ones based on quality scores compared to the pseudo-quality labels. 
Second, we qualitatively analyze the impact of using auxiliary anomalies in training a model, showing the learning curves (\textbf{LC$_{\textsc{g}}$}) with the number of added anomalies following the ranking of quality scores on the x-axis and the test accuracy on the y-axis. We compute the area under the learning curve up to $\frac{1}{3}$ of ranked anomalies (\textbf{AULC$_{\textsc{g}}$}) and the test accuracy after including top $\frac{1}{3}$ of ranked anomalies (\textbf{ACC$_{\textsc{g}}$}). Similarly, we compute the \textbf{LC$_{\textsc{p}}$} following the methods' inverse ordering and measure the \textbf{AULC$_{\textsc{p}}$} of including up to $\frac{2}{3}$ of inversely-ranked anomalies, where lower values are desirable. Also, rankings from $1$ (best) to $13$ (worst) are assigned for each method in every experiment, which are denoted by a $r$ in front in Table~\ref{tab:aggregated_results}.

\begin{table*}[htpb]
\centering
\small
\caption{Summary of the results obtained by the $13$ methods over $18$ image (above) and $22$ tabular (below) datasets. Columns $6-10$ show the ranking values for the $4$ metrics employed (columns $2-5$) and their average. For metrics desiring lower values, we mark with a $\downarrow$. Overall, \ourmethod{} achieves the best performance and ranking position for most evaluation metrics as well as the best avg. ranking.}\label{tab:aggregated_results}
\resizebox{\linewidth}{!}{ 
\begin{tabular}{l|c|c|c|c||c|c|c|c||c}
\multicolumn{10}{c}{\bf{\textsc{18 Image Datasets}}}\\
\toprule
\textsc{Evaluator} & \textbf{AUC$_{\textsc{qlt}}$} & \textbf{AULC$_{\textsc{g}}$} & \textbf{ACC$_{\textsc{g}}$} & \textbf{AULC$_{\textsc{p}}$}($\downarrow$) 
& \textbf{rAUC$_{\textsc{qlt}}$} & \textbf{rAULC$_{\textsc{g}}$} & \textbf{rACC$_{\textsc{g}}$} & \textbf{rAULC$_{\textsc{p}}$}  & \textbf{\textsc{Avg. Rank}} \\
\midrule
\bf{\ourmethod{}} & \bf{0.803}  & \bf{0.717} & \bf{0.833} & 0.688 &  \bf{1.99}  & \bf{3.94} & \bf{3.12} & \bf{3.66}  & \bf{3.18} \\
\textsc{Rarity} & 0.681  & 0.698 & 0.786 & 0.738  & 4.21 & 5.07 & 5.34 & 7.96 & 5.65 \\
\textsc{Lava} & 0.742  & 0.665 & 0.755 & 0.709  & 2.70  & 8.17 & 7.44 & 4.71 & 5.75 \\
$P_{y|x}+NP_x$ & 0.669 & 0.700 & 0.795 & 0.729 & 4.80 & 5.60 & 5.61 & 7.13 & 5.79 \\
\textsc{Loo} & 0.537  & 0.693 & 0.777 & 0.694  & 7.22& 5.86 & 6.56 & 6.14  & 6.44 \\
\textsc{RandomEv} & 0.491 & 0.696 & 0.793 & 0.756 & 8.92  & 6.21 & 5.93 & 9.90 & 7.38 \\
\textsc{DataOob} & 0.505 & 0.685 & 0.739 & \bf{0.668} & 8.66 & 7.53 & 9.55 & 4.35& 7.52 \\
\textsc{kNNShap} & 0.509  & 0.695 & 0.794 & 0.754  & 8.58 & 6.36 & 5.95 & 9.44 & 7.58 \\ 
\textsc{AME} & 0.512 & 0.693 & 0.793 & 0.756& 8.56 & 6.42 & 5.85 & 9.94 &  7.69 \\
\textsc{Inf} & 0.488  & 0.681 & 0.778 & 0.744  & 8.87 & 7.01 & 6.73 & 8.55 &  7.79 \\
$P_{y|x}$ & 0.494 & 0.671 & 0.710 & 0.644 & 9.30 & 8.55 & 10.62 & 3.04 & 7.88 \\
\textsc{DataBanzhaf} & 0.500 & 0.668 & 0.775 & 0.748  & 8.71  & 8.01 & 6.68 & 8.54  & 7.98 \\
$P_x$ & 0.502 & 0.535 & 0.600 & 0.733 & 8.59 & 12.37 & 11.72 & 7.72 & 10.10 \\
\midrule
\multicolumn{10}{c}{\bf{\textsc{22 Tabular Datasets}}}\\
\midrule
\bf{\ourmethod{}} & \bf{0.821} & 0.779 & \bf{0.839} & \bf{0.717}  & \bf{1.91}  & \bf{4.27} & \bf{4.28} & \bf{2.11} & \bf{3.14} \\
\textsc{Rarity} & 0.724 & \bf{0.782} & \bf{0.839} & 0.753  & 3.59  & 4.42 & 4.91 & 4.92 & 4.46 \\
\textsc{Lava} & 0.723 & 0.744 & 0.794 & 0.747 & 3.62  & 7.48 & 7.17 & 4.08  & 5.59 \\
$P_{y|x}+NP_x$ & 0.676 & 0.758 & 0.815 & 0.771 & 4.93 & 6.16 & 6.23 & 6.16 & 5.87 \\
\textsc{kNNShap} & 0.541  & 0.770 & 0.825 & 0.805  & 7.61  & 5.24 & 5.29 & 9.15 & 6.82 \\
\textsc{RandomEv} & 0.502 & 0.773 & 0.827 & 0.809  & 8.25  & 4.92 & 5.33 & 9.72  & 7.06 \\
\textsc{AME} & 0.498  & 0.772 & 0.827 & 0.809  & 8.53  & 5.05 & 5.37 & 9.75  & 7.17 \\
\textsc{Loo} & 0.504  & 0.750 & 0.798 & 0.792 & 7.90 & 6.87 & 7.10 & 7.77  & 7.41 \\
$P_x$ & 0.554 & 0.703 & 0.768 & 0.752 & 6.77 & 10.39 & 9.29 & 4.69 & 7.78 \\
$P_{y|x}$ & 0.533 & 0.712 & 0.748 & 0.753 & 8.20 & 10.56 & 10.66 & 4.30 & 8.43 \\
\textsc{DataOob} & 0.513 & 0.729 & 0.768 & 0.785 & 8.48  & 9.39 & 9.84 & 7.43  & 8.78 \\
\textsc{Inf} & 0.434  & 0.742 & 0.795 & 0.816  & 10.29  & 7.78 & 7.37 & 10.40  & 8.96 \\
\textsc{DataBanzhaf} & 0.415  & 0.736 & 0.789 & 0.817  & 10.92  & 8.47 & 8.16 & 10.53 & 9.52 \\
\bottomrule
\end{tabular}
}
\vspace{-.1in}
\end{table*}
\subsection{Experimental Results}

\paragraph{Q1. \ourmethod{} vs baselines at assigning quality scores.}

Figure~\ref{fig:q1_auroc_qs} shows the methods' mean AUC$_{\textsc{qlt}}$ on both image (left) and tabular data (right). On images, \ourmethod{} outperforms all baselines on $13$ out of $18$ datasets, achieving an average AUC$_{\textsc{qlt}}$ significantly higher than \textsc{Lava} and \textsc{Rarity} by $6$ and $12$ percentage points, as shown in Table~\ref{tab:aggregated_results}. Also, \ourmethod{} consistently obtains the lowest (best) average ranking positions ($1.99$ for rAUC$_{\textsc{qlt}}$). On tabular data, \ourmethod{} 
obtains an average AUC$_{\textsc{qlt}} = 0.821$, which is around $10$ percentage points higher than the runner-ups. Also, \ourmethod{} outperforms \textsc{Lava} and \textsc{Rarity} on $18$ and $17$ datasets and obtains the best average ranking ($1.91$).


Interestingly, only \ourmethod{}, \textsc{Lava}, \textsc{Rarity}, and $P_{y|x}+NP_x$ achieve performance better than random, while other methods get average AUC$_{\textsc{qlt}}$ around $0.5$, which highlights their inability to distinguish good and poor auxiliary anomalies consistently. As a second remark, \ourmethod{} performs lower than random for the dataset \textsc{Tile}. This happens because the defective images are extremely different than the normal images, thus resulting in real anomalies falling in zero-density regions, which our method would categorize as unrealistic.

\begin{figure*}[t!]
\centerline{\includegraphics[width=\textwidth]{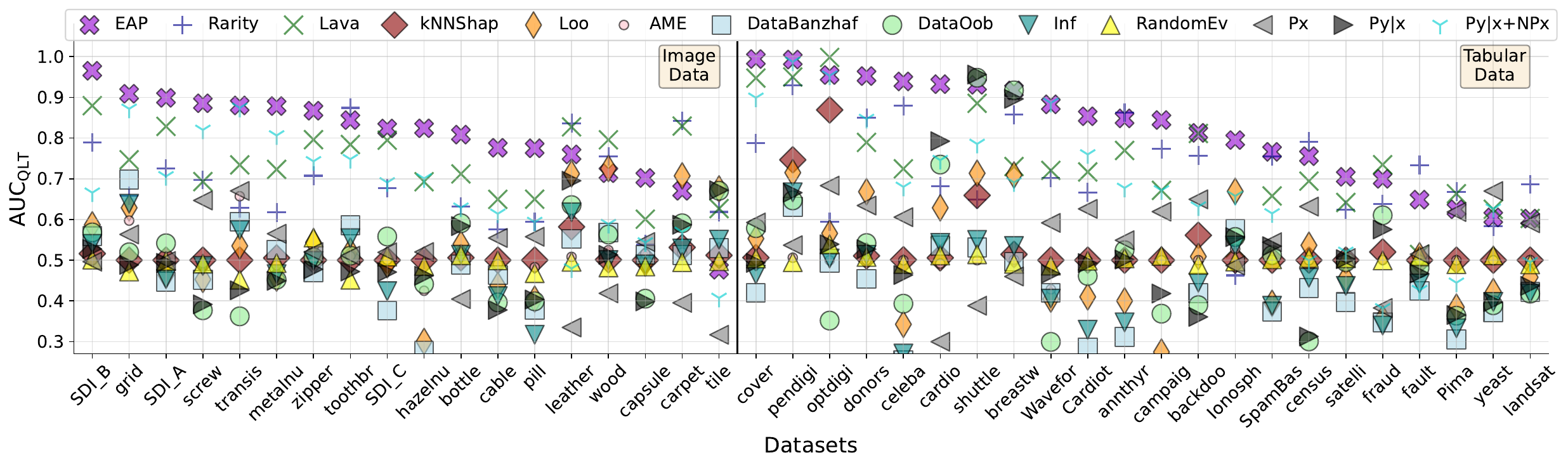}}
\vspace{-.12in}
\caption{The plot illustrates the average AUC$_{\textsc{qlt}}$ obtained by each method on a per-dataset basis (left for image data, right for tabular data). \ourmethod{} achieves the highest (best) performance for most datasets, beating the runner-ups \textsc{Rarity} and \textsc{Lava} on, respectively, $30$ and $31$ datasets out of $40$.
}\label{fig:q1_auroc_qs}
\vspace{-.03in}
\end{figure*}

\begin{figure*}[t!]
\centerline{\includegraphics[width=\textwidth]{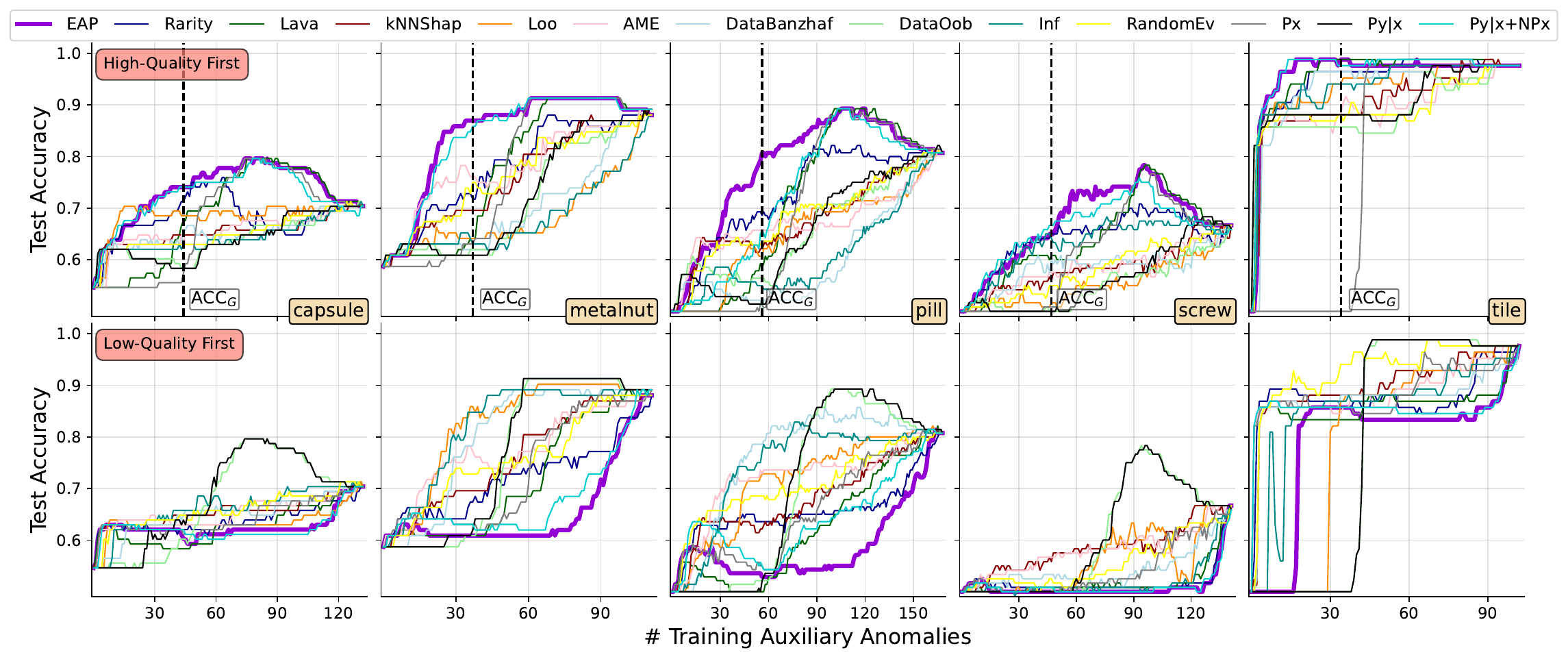}}
\vspace{-.12in}
\caption{Learning curve (LC) obtained by following the method's ordering (top) and inverse ordering (bottom) for five representative image datasets. Top: \ourmethod{}' LC$_{\textsc{g}}$ grows sooner (i.e., better) than the other methods', which confirms that including high-quality anomalies in the training set has a larger impact on the test performance. Bottom: \ourmethod{}' LC$_{\textsc{p}}$ rises later (i.e., better) than most baselines', showing that low-quality anomalies have a comparatively modest impact on the test performance.
}\label{fig:q2_aucl}
\vspace{-.03in}
\end{figure*}

\paragraph{Q2. Including high-quality anomalies in training.}

We measure how the performance of a model varies when introducing the top $\frac{1}{3}$ of auxiliary anomalies into the training following the methods' rankings. Figure~\ref{fig:q2_aucl} (top) shows the learning curves (LC$_{\textsc{g}}$) for five representative image datasets. Overall, including high-ranked anomalies first has the claimed impact on the model's performance: the learning curve grows sooner for \ourmethod{} compared to all baselines. Consequently,  \ourmethod{} obtains an AULC$_{\textsc{g}}$ that is, on average, higher than all the baselines by between $2$ (vs. \textsc{Rarity} and $P_{y|x}+NP_x$) and $5$ (vs. \textsc{Lava}) percentage points. After including one-third of the auxiliary anomalies for training, \ourmethod{} shows an average improvement on the test performance (ACC$_{\textsc{g}}$) of $4$ to $10$ points over all baselines, standing out as the only method significantly better than \textsc{RandomEv}.
Moreover, Table~\ref{tab:aggregated_results} shows that \ourmethod{} achieves the best average ranking for both AULC$_{\textsc{g}}$ (i.e., $3.94$) and ACC$_{\textsc{g}}$ (i.e., $3.12$).

On tabular data, Table~\ref{tab:aggregated_results} shows that \ourmethod{} achieves an average AULC$_{\textsc{g}}$ slightly lower than \textsc{Rarity} ($0.779$ vs $0.782$) and a similar ACC$_{\textsc{g}}$ (both $0.839$). This occurs because, for most tabular datasets, \textsc{Rarity} assigns high-quality scores to the unrealistic anomalies which are obviously different than normal data. When including them for training, it yields an improvement of the \textsc{Random Forest}'s test accuracy surprisingly. From a ranking perspective, \ourmethod{} obtains the lowest (best) average with rAULC$_{\textsc{g}} = 4.27$ (vs. \textsc{Rarity}'s $4.42$), and rACC$_{\textsc{g}} = 4.28$ (vs. \textsc{Rarity}'s $4.91$).

\begin{table*}[htpb]
\setlength{\tabcolsep}{3pt}
\caption{Test AUCs ($\%$) of prompt tuning with auxiliary anomalies selected by each baseline on \textsc{MvTec}. \ourmethod{} performs the best on $12$ of $15$ classes and achieves the highest average AUC.}
\label{tab:prompt_tuning}
    \small
    \centering
     \resizebox{\linewidth}{!}{ 
    \begin{tabular}{lccccccccccccccc|c}
        \toprule
Evaluator       &bottle&	cable&	capsule&	carpet&	grid	&hazel	&leather	&metal	&pill	&screw	&tile	&tooth	&transis	&wood	&zipper	&avg\\
\midrule
\textsc{KNNShap}	&82.4	&56.7	&51.4	&82.1	&58.8	&72.7	&92.1	&51.7	&49.1	&60.2	&69.2	&\bf77.0	&\bf77.1	&72.7	&44.7	&66.5\\
\textsc{AME}	&82.4	&56.7	&51.4	&82.1	&71.2	&72.7	&98.4	&51.7	&49.1	&60.2	&69.2	&75.4	&68.3	&72.7	&44.7	&67.1\\
\textsc{LOO}	&82.4	&60.7	&64.4	&85.8	&58.8	&68.2	&\bf98.8	&51.7	&49.1	&60.4	&69.2	&\bf77.0	&62.1	&72.7	&44.7	&67.1\\
\textsc{DataOob}	&82.4	&60.7	&63.7	&85.8	&71.2	&72.7	&\bf98.8	&51.7	&49.1	&60.4	&69.2	&75.4	&62.1	&72.7	&44.7	&68.0\\
$P_{y|x}$	&82.4	&60.6	&51.2	&85.8	&71.2	&68.2	&98.4	&56.2	&65.6	&60.4	&69.2	&75.4	&62.1	&72.7	&44.7	&68.3\\
\textsc{INF}	&\bf95.3	&56.7	&45.7	&\bf96.1	&83.6	&68.2	&\bf98.8	&51.7	&\bf65.7	&60.4	&69.2	&75.4	&62.1	&93.0	&\bf75.9	&73.2\\
\textsc{DataBanz}	&\bf95.3	&41.1	&64.4	&\bf96.1	&83.6	&72.7	&\bf98.8	&51.7	&65.7	&60.4	&84.9	&75.4	&62.1	&93.0	&\bf75.9	&74.7\\
\textsc{Randomev}	&\bf95.3	&69.2	&63.7	&\bf96.1	&83.1	&68.2	&\bf98.8	&56.2	&\bf65.7	&60.4	&87.8	&75.4	&62.1	&\bf96.4	&\bf75.9	&76.9\\
\textsc{Rarity}	&94.3	&69.2	&64.4	&\bf96.1	&83.6	&\bf83.8	&\bf98.8	&45.1	&63.5	&60.4	&\bf91.8	&75.4	&\bf77.1	&\bf96.4	&\bf75.9	&78.4\\
$P_x$	&94.3	&70.2	&\bf64.6	&\bf96.1	&83.6	&76.7	&86.9	&90.8	&63.5	&60.4	&84.9	&75.4	&\bf77.1	&93.0	&\bf75.9	&79.6\\
\textsc{Lava}	&94.3	&\bf70.3	&64.4	&\bf96.1	&83.6	&\bf83.8	&98.4	&\bf90.9	&63.5	&60.0	&87.8	&75.4	&\bf77.1	&93.0	&\bf75.9	&80.9\\
$P_{y|x}+NP_x$	&94.3	&70.2	&\bf64.6	&\bf96.1	&\bf92.6	&\bf83.8	&94.4	&90.8	&63.5	&60.4	&87.5	&\bf77.0	&\bf77.1	&93.0	&\bf75.9	&81.4\\
\midrule
\textbf{\ourmethod{}}	&\bf95.3	&\bf70.3	&64.4	&\bf96.1	&\bf92.6	&\bf83.8	&\bf98.8	&\bf90.9	&63.5	&\bf64.8	&\bf91.8	&\bf77.0	&\bf77.1	&93.0	&\bf75.9	&\bf82.4\\

\bottomrule  
\end{tabular}
}
\vspace{-.1in}
\end{table*}

\paragraph{Q3. Including low-quality anomalies in training.}
Figure~\ref{fig:q2_aucl} (bottom) shows the LC$_{\textsc{p}}$ obtained by following an inverse ordering of the methods, i.e., lower ranked anomalies are included first. Using this inverse ordering should result in much slower growth of the LC$_{\textsc{p}}$s: in some cases, the test accuracy using \ourmethod{} remains stable (\textsc{Capsule}, \textsc{Metalnut}, \textsc{Screw}), while in others it shows strong fluctuations going up and down quickly (\textsc{Pill}). Interestingly, every baseline's performance goes up for \textsc{Tile}: this is due to their poor ability to assign scores for this dataset, as described in Q1. Surprisingly, \textsc{DataOob} obtains the lowest (best) AULC$_{\textsc{p}}$, while \ourmethod{} has the second best AULC$_{\textsc{p}}$ with just two percentage points as gap. However, when ranking the experiments, \ourmethod{} achieves the best average ranking ($3.66$ of rAULC$_{\textsc{p}}$), thus being the preferred method for most of the experiments. 

On tabular data, the results confirm the previous analysis: \ourmethod{} has the lowest AULC$_{\textsc{p}}$ on $137$ experiments out of $220$, while \textsc{Rarity} and \textsc{Lava} achieve so only on, respectively, $29$ and $31$ experiments. This motivates that \ourmethod{} obtains an average AULC$_{\textsc{p}}$ that is better than all baselines by between $3$ and $10$ percentage points, as shown in Table~\ref{tab:aggregated_results}.

\paragraph{Q4. Prompt tuning for zero-shot anomaly detection.}
CLIP-based anomaly detection methods save the effort of collecting training examples and enable a zero-shot anomaly detection~\cite{jeong2023winclip}. However, their detection performance depends on the choice of prompts, which is usually tuned by using labeled real-world anomalies.
We study the impact of selected auxiliary anomalies on prompt tuning for the \textsc{MvTec} datasets.
Specifically, we search a prompt for each object class achieving the best performance on the selected auxiliary anomalies from a pool of $27$ candidate prompts (see details in \Cref{app:prompts}), and apply the best-performing prompt to CLIP at test time.

\Cref{tab:prompt_tuning} reports the test AUCs of CLIP with the best-performing prompts selected by each data valuation method. We can see that \ourmethod{} performs the best on $12$ of $15$ classes and achieves the highest AUC averaged over all classes. Thanks to accounting for both the class conditional probability and the data density, \ourmethod{} clearly outperforms \textsc{Rarity} and $P_x$, which only consider the data density, $P_{y|x}$, which only considers the class conditional probability, and their naive linear combination $P_{y|x}+NP_x$.
The results confirm that \ourmethod{} selects high-quality auxiliary anomalies for the model selection purpose. We list the prompts selected by \ourmethod{} in \Cref{app:prompts}.

\paragraph{Q5. \ourmethod{}' sensitivity to $\frac{\alpha_0}{\alpha_0+\beta_0}$.} 


\ourmethod{} requires two hyperparameters: $\alpha_0$, $\beta_0$ of the prior Beta distribution. To remove one degree of freedom, we set $\alpha_0+\beta_0 = 1$ such that the dataset size $n$ is much stronger ($n$ times) than our initial belief. Then, we investigate how varying the parameter $\alpha_0 \in [0,0.5)$ impacts \ourmethod{}'s overall performance ($\beta_0 = 1-\alpha_0$). We compare seven versions of our method by setting $\alpha_0 \in \{0.01,0.05,0.1,0.2,0.3,0.4\}$, in addition to the original \ourmethod{} that leverages the contamination level $\alpha_0 = m/n$. We call \ourmethod{}$_{w}$ the variant that uses $\alpha_0 = w$. \Cref{tab:hyperparams} in the Supplement shows the rankings of these $7$ variants for AUC$_{\textsc{qlt}}$ and ACC$_{\textsc{g}}$ and their average. Overall, both parameters have a low impact on our method: while a higher value for $\alpha_0$ improves the AUC$_{\textsc{qlt}}$, in some experiments this improvement does not yield better performance at test time in terms of ACC$_{\textsc{g}}$. Moreover, setting $\alpha_0$ too high or too low has inherent risks: \ourmethod{}$_{0.4}$ and \ourmethod{}$_{0.01}$ are the worst variants by far, with significant drops in performance compared to the other variants.
\section{Conclusion and Limitations}\label{sec:conclusion_and_limitations}
This paper addressed the problem of evaluating the quality of an auxiliary set of synthetic anomalies. With this quality score, one can enrich an anomaly detection dataset to learn a more accurate anomaly detector. 
We proposed the expected anomaly posterior (\ourmethod{}), the first quality score function for auxiliary anomalies derived from an approximation for the posterior over the probability that a given input is an anomaly.
We showed that our approach theoretically assigns higher scores to the realistic anomalies, compared to unrealistic and indistinguishable anomalies. Empirically, we investigated how \ourmethod{} compares to adapted data quality estimators at (1) assigning quality scores, (2) using such scores to enrich the data for training, and (3) model selection. On $40$ datasets, we show that \ourmethod{} outperforms all $12$ baselines in the majority of the cases.

\paragraph{Limitations.} While our theoretical categorization refers to general anomalies, the concept of unrealistic/indistinguishable might be domain-driven and require adjustments in some applications (e.g., scratches on fabrics are unrealistic but would get high quality). Also, uncertainty-based methods require limited training samples to be effective and one may need to reduce the training normals.

\paragraph{Impact Statement.} This paper presents work whose goal is to advance the field of Machine Learning. There are many potential societal consequences of our work, none of which we feel must be specifically highlighted here.

\appendix
\bibliography{bibliography}
\newpage
\section{Supplement}
\subsection{Data quality estimators.}
\label{sec:data_quality_estimators}

Any data quality estimator can be seen as a mapping that assigns a scalar score to any example $(x,y)$. Such a score quantifies the impact of $(x,y)$ on the model's performance when trained including the example in the training set. For this task, they introduce a utility function $U(\bar{D}) \coloneqq \textsc{Perf}(f, \bar{D})$ that takes as input a subset $\bar{D}$ of $D$ and measures the performance of $f$ when trained on it.
Next, we briefly describe the existing data quality estimators employed in the experiments and refer to~\citep{jiang2023opendataval} for additional details.

\begin{itemize}[nolistsep,leftmargin=*]
    \item \textsc{Leave One Out (Loo)} is defined as $\phi_{\textsc{loo}}(x,y) = U(D) - U(D\backslash\{(x,y)\})$, where $U$ is commonly chosen as the accuracy;
    \item \textsc{DataShap} generalizes \textsc{Loo}'s approach to the concept of marginal contributions, which measures the average change in utility when $(x,y)$ is removed from any training set. Given a training set cardinality $j \le N$, the marginal contribution is defined as
    \begin{equation*}
        \mathcal{M}_j(x,y) \coloneqq \binom{N-1}{j-1}^{-1} \sum_{\bar{D}_j \subseteq D, |\bar{D}_j|=j-1} U(\bar{D}_j\cup \{(x,y)\}) - U(\bar{D}_j)
    \end{equation*}
    where $\bar{D}_j$ is a random subset of $D$ of cardinality $j-1$ that does not contain $(x,y)$. Then, \textsc{DataShap}~\cite{ghorbani2019data} computes the score as $\phi_{\textsc{DataShap}}(x,y) = \frac{1}{N} \sum_{j=1}^N \mathcal{M}_j(x,y)$;
    \item \textsc{BetaShap}~\cite{kwon2021beta} generalizes \textsc{DataShap} by considering a weighted average of marginal contributions $\phi_{\textsc{BetaShap}}(x,y) = \frac{1}{N} \sum_{j=1}^N \omega_j \mathcal{M}_j(x,y)$, for some weights $\omega_1,\dots,\omega_N$.
    \item \textsc{DataBanzhaf}~\cite{wang2023data} exploits the same formulation as \textsc{BetaShap} but sets the weights to $\omega_j = 2^{-N} \binom{N-1}{j-1}$.
    \item \textsc{AME}~\cite{lin2022measuring} shows that the average marginal contribution taken over random subsets of $D$ can be efficiently estimated by predicting the model's prediction. They employ a LASSO regression model that minimizes
    \begin{equation*}
        \argmin_{\gamma \in \mathbb{R}^N} \mathbb{E}\left[U(\bar{D}) - g(\mathds{1}(\bar{D}))^T\gamma\right]^2 + \lambda \sum_{i=1}^N |\gamma_i|,
    \end{equation*}
    where $\mathds{1}(\bar{D})$ is the multi-dimensional characteristic function, $\bar{D}$ is a random subset draw the data distribution, $\lambda$ is the regularization parameter, and $g\colon \{0,1\} \to \mathbb{R}^N$ is a predefined transformation. The values $\gamma_i$ represent the quality of $(x_i,y_i)$.
    \item \textsc{kNNShap}~\cite{jia2019efficient} differs from \textsc{DataShap} on the choice of the utility function:
    \begin{equation*}
        U(\bar{D}) = \frac{1}{N_{val} k} \sum_{i=1}^{N_{val}} \sum_{(x_j,y_j) \in \mathcal{N}(x_i, \bar{D})} \mathds{1}(y_i = y_j),
    \end{equation*}
    where $k$ is the number of neighbors, $N_{val}$ is the size of the validation set, and $\mathcal{N}(x_i, \bar{D})$ indicates the set of nearest neighbors for the validation example $x_i$ over the subset $\bar{D}$. Roughly speaking, it measures the proportion of examples in $\bar{D}$ that are neighbors of $x_i$ and share the same label $y_i$.
    \item \textsc{Influence Functions (Inf)}~\cite{feldman2020neural} approximate the difference of utility functions in \textsc{Loo} by splitting $D$ into two subsets of equal cardinalities and randomly drawing subsets from each of them:
    \begin{equation*}
        \phi_{\textsc{Inf}}(x,y) = \mathbb{E}_{\bar{D}_x} [U(\bar{D}_x)] - \mathbb{E}_{\bar{D}_{\not x}} [\bar{D}_{\not x}],
    \end{equation*}
    where all the subsets from $\bar{D}_x$ contain $(x,y)$, while none of the subsets from $\bar{D}_{\not x}$ contain $(x,y)$.
    \item \textsc{Lava}~\cite{just2023lava} measures the quality of $(x,y)$ by quantifying how fast the optimal transport cost between the training and validation sets changes when increasing more weight to $(x,y)$. That is,
    \begin{equation*}
        \phi_{\textsc{Lava}}(x,y) = h^* - \frac{1}{N-1} \sum_j h^*_j,
    \end{equation*}
    where $h^*_i$ is part of the optimal solution of the transport problem.
    \item \textsc{DataOob}~\cite{kwon2023data} relies on the concept of out-of-bag estimate to capture the data quality. Given $B$ weak learners $f_b$, each trained on a bootstrap sample of $D$, the quality score is
    \begin{equation*}
        \phi_{\textsc{DataOob}}(x,y) = \frac{\sum_{b=1}^B \mathds{1}(w_{b}=0) T(y, f_b(x))}{\sum_{b=1}^B \mathds{1}(w_{b}=0)},
    \end{equation*}
    where $w_{b}$ is the number of times $(x,y)$ is selected in the $b-$th bootstrap, and $T$ is an evaluation metric (e.g., correctness $\mathds{1}(y = f_b(x))$).
\end{itemize}

\subsection{Rarity score}
\label{sec:rarity_score}
Formally, given a synthetic image with extracted feature $x$, the rarity score is a function $r_k \colon \R^d \to \R$ such that 
\begin{equation}
r_k(x) = 
\begin{cases}
    0  &\text{if } x \notin \!\!\bigcup\limits_{x_i \in D} \!\! B_k(x_i)\\
    \min\limits_{x_i \in D \colon x \in B_k(x_i)} \!\! NN_k(x_i) &\text{otherwise}
\end{cases}
\end{equation}
where $NN_k(x_i)$ is the distance between $x_i$ and its k-th nearest neighbor in $D$, and $B_k(x_i) = \{x|d(x_i,x) \le NN_k(x_i)\}$ is the k-NN sphere with $x_i$ as center and $NN_k(x_i)$ as radius. The rarity score strongly depends on the choice of the hyperparameter $k$: high values of $k$ could map far unrealistic examples to a positive high score, namely they would be considered authentic, while low values of $k$ could map real examples slightly different than the training data to a null score, namely they would be considered artifacts.

Because the rarity score strictly depends on the hyperparameter $k \in \{1,\dots, n-1\}$, we need to estimate a proper value $\hat{k}$. Let's assume the existence of an optimal $k$, and use the small set of $m$ anomalies to estimate it.
Ideally, $k$ should be: (1) as low as possible to assign null scores to unrealistic anomalies, and (2) high enough to assign positive scores to the real training anomalies. 

Following this insight, we assume a Bayesian perspective and set a normalized variable's $K$ prior to be uniform
\begin{equation*}
    K \coloneqq \frac{k-1}{n-1} \sim \textsc{Beta}(1,1) = \textsc{Unif}(0,1).
\end{equation*}
Roughly speaking, we min-max normalize $K$ to $[0,1]$ to exploit that a Beta prior with a Bernoulli likelihood results in a Beta posterior distribution. Because we want the minimum $k$ that assigns positive scores to the training anomalies $\{x_{\bar{m}}\}_{\bar{m}\le m}$, we compute for each $x_{\bar{m}}$ the minimum $k_{\bar{m}}$ such that $r_{k_{\bar{m}}}(x_{\bar{m}})>0$. The set of normalized $\{\frac{k_{\bar{m}}-1}{n-1}\}_{\bar{m}\le m}$ is the empirical evidence for the Bayesian update, which is
\begin{equation*}
    K\Big|\!\left\{\!\frac{k_{\bar{m}}\!-\!1}{n\!-\!1}\!\right\}\! \sim \!\textsc{Beta}\!\!\left(\!\!1\!+\!\!\!\sum_{\bar{m}\le m}\!\! \frac{k_{\bar{m}}\!-\!1}{n\!-\!1}, 1\! +\! m\! -\!\!\! \sum_{\bar{m}\le m} \!\!\frac{k_{\bar{m}}\!-\!1}{n\!-\!1}\!\right)\!.
\end{equation*}
Finally, we estimate $\hat{k}$ as the $95$th percentile of the posterior distribution of $K$
\begin{equation}
    \hat{k} = \argmin_{t\in[0,1]} \Pp\left(K\Big|\left\{\frac{k_{\bar{m}}-1}{n-1}\right\} \le t\right) \ge 0.95
\end{equation}
which guarantees that at least $95\%$ of real anomalies get a positive rarity score.

\subsection{Prompt tuning for CLIP}\label{app:prompts}
\begin{tcolorbox}
\begin{scriptsize}
\begin{verbatim}
%Candidate prompt templates for MvTec:
    [`{}',`damaged {}'],
    [`flawless {}',`{} with flaw'],
    [`perfect {}',`{} with defect'],
    [`unblemished {}',`{} with damage'],
    [`{} without flaw',`{} with flaw'],
    [`{} without defect',`{} with defect'],
    [`a photo of a normal {}',`a photo of an anomalous {}'],
    [`a cropped photo of a normal {}', `a cropped photo of an anomalous {}'],
    [`a dark photo of a normal {}', `a dark photo of an anomalous {}'],
    [`a photo of a normal {} for inspection', `a photo of an anomalous {} for inspection'],
    [`a photo of a normal {} for viewing', `a photo of an anomalous {} for viewing'],
    [`a bright photo of a normal {}', `a bright photo of an anomalous {}'],
    [`a close-up photo of a normal {}', `a close-up photo of an anomalous {}'],
    [`a blurry photo of a normal {}', `a blurry photo of an anomalous {}'],
    [`a photo of a small normal {}', `a photo of a small anomalous {}'],
    [`a photo of a large normal {}', `a photo of a large anomalous {}'],
    [`a photo of a normal {} for visual inspection', `a photo of an anomalous {} for visual inspection'],
    [`a photo of a normal {} for anomaly detection',`a photo of an anomalous {} for anomaly detection'],
    [`a photo of a {}',`a photo of someting'],
    [`a cropped photo of a {}', `a cropped photo of someting'],
    [`a dark photo of a {}', `a dark photo of someting'],
    [`a photo of a {} for inspection', `a photo of someting for inspection'],
    [`a bright photo of a {}', `a bright photo of someting'],
    [`a close-up photo of a {}', `a close-up photo of someting'],
    [`a blurry photo of a {}', `a blurry photo of someting'],
    [`a photo of a {} for visual inspection', `a photo of someting for visual inspection'],
    [`a photo of a {} for anomaly detection',`a photo of someting for anomaly detection']
\end{verbatim}
\end{scriptsize}
\end{tcolorbox}

\begin{tcolorbox}
\begin{scriptsize}
\begin{verbatim}
%EAP selected prompts for MvTec:
    [`bottle', `damaged bottle'],
    [`cable without defect', `cable with defect'],
    [`unblemished capsule', `capsule with damage'],
    [`carpet', `damaged carpet'],
    [`a bright photo of a normal grid', `a bright photo of an anomalous grid'],
    [`hazelnut without defect', `hazelnut with defect'],
    [`a photo of a normal leather for inspection', `a photo of an anomalous leather for inspection'],
    [`metalnut', `damaged metalnut'],
    [`pill', `damaged pill'],
    [`a close-up photo of a screw', `a close-up photo of someting'],
    [`tile', `damaged tile'],
    [`toothbrush without flaw', `toothbrush with flaw'],
    [`a blurry photo of a normal transistor', `a blurry photo of an anomalous transistor'],
    [`wood', `damaged wood'],
    [`zipper without defect', `zipper with defect']
\end{verbatim}
\end{scriptsize}
\end{tcolorbox}

\begin{table}[htpb]
\centering
\small
\caption{Comparison between \ourmethod{} with default $\alpha_0$ and its six variants \ourmethod{}$_w$ that set $\alpha_0=w, \beta_0=1-w$. Rankings show low sensitivity to such a choice, as long as $\alpha_0<0.4$.}\label{tab:hyperparams}
\begin{tabular}{lcccc}
\toprule
\textsc{Evaluator} & \textbf{rAUC$_{\textsc{qlt}}$} & \textbf{rACC$_{\textsc{g}}$}  & \textbf{\textsc{Avg. Rank}} \\
\midrule
\ourmethod{} & 2.92 & \bf{2.73} & \bf{2.83} \\
\ourmethod{}$_{0.2}$ & \bf{2.78} & 3.25 & 3.02 \\
\ourmethod{}$_{0.3}$ & 3.39 & 3.40  & 3.40 \\
\ourmethod{}$_{0.1}$ & 3.01 & 3.46  & 3.24 \\
\ourmethod{}$_{0.05}$ & 3.63 & 3.78  & 3.71 \\
\ourmethod{}$_{0.01}$ & 4.10 & 3.54 & 3.82 \\
\ourmethod{}$_{0.4}$ & 4.71 & 4.79 & 4.75 \\
\bottomrule
\end{tabular}
\end{table}

\begin{table*}[htpb]
\setlength{\tabcolsep}{2.5pt}
\centering
\small
\caption{Summary of the results obtained by the $13$ methods over all $40$ datasets. We report the mean $\pm$ std, computed over all experiments. Overall, \ourmethod{} achieves the best performance and ranking position for all evaluation metrics as well as the best average ranking (last column).}\label{tab:mean_pm_std}
\resizebox{\linewidth}{!}{ 
\begin{tabular}{l|c|c|c|c|c|c|c|c|c}
\toprule
 \textsc{Evaluator} & \textbf{AUC$_{\textsc{qlt}}$} & \textbf{AULC$_{\textsc{g}}$} & \textbf{ACC$_{\textsc{g}}$} & \textbf{AULC$_{\textsc{p}}$}($\downarrow$) 
& \textbf{rAUC$_{\textsc{qlt}}$} & \textbf{rAULC$_{\textsc{g}}$} & \textbf{rACC$_{\textsc{g}}$} & \textbf{rAULC$_{\textsc{p}}$}  & \textbf{\textsc{Avg. Rank}} \\
\midrule   
\bf{\ourmethod{}} & \bf{0.81 $\pm$ 0.13} & \bf{0.76 $\pm$ 0.15} & \bf{0.84 $\pm$ 0.14} & \bf{0.70 $\pm$ 0.14} & \bf{1.95 $\pm$ 1.65} & \bf{4.12 $\pm$ 2.88} & \bf{3.76 $\pm$ 2.75} & \bf{2.82 $\pm$ 2.17} & \bf{3.16 $\pm$ 1.67} \\
\textsc{Rarity} & 0.70 $\pm$ 0.14 & 0.74 $\pm$ 0.16 & 0.82 $\pm$ 0.16 & 0.75 $\pm$ 0.15 & 3.87 $\pm$ 2.73 & 4.71 $\pm$ 3.36 & 5.11 $\pm$ 3.51 & 6.29 $\pm$ 3.49 & 4.99 $\pm$ 2.68 \\
\textsc{Lava} & 0.73 $\pm$ 0.13 & 0.71 $\pm$ 0.16 & 0.78 $\pm$ 0.17 & 0.73 $\pm$ 0.14 & 3.21 $\pm$ 2.03 & 7.79 $\pm$ 3.63 & 7.29 $\pm$ 3.62 & 4.36 $\pm$ 2.54 & 5.66 $\pm$ 2.19 \\
$P_{y|x}+NP_x$ & 0.67 $\pm$ 0.18 & 0.73 $\pm$ 0.16 & 0.81 $\pm$ 0.16 & 0.75 $\pm$ 0.14 & 4.82 $\pm$ 3.22 & 5.87 $\pm$ 3.40 & 5.90 $\pm$ 3.37 & 6.55 $\pm$ 3.22 & 5.79 $\pm$ 2.53 \\
\textsc{Loo} & 0.52 $\pm$ 0.17 & 0.72 $\pm$ 0.16 & 0.79 $\pm$ 0.16 & 0.75 $\pm$ 0.14 & 7.59 $\pm$ 3.48 & 6.42 $\pm$ 3.57 & 6.86 $\pm$ 3.31 & 7.04 $\pm$ 3.42 & 6.98 $\pm$ 2.90 \\
\textsc{kNNShap} & 0.53 $\pm$ 0.08 & 0.74 $\pm$ 0.16 & 0.81 $\pm$ 0.15 & 0.78 $\pm$ 0.15 & 8.05 $\pm$ 2.02 & 5.74 $\pm$ 2.84 & 5.59 $\pm$ 2.60 & 9.28 $\pm$ 2.36 & 7.16 $\pm$ 1.40 \\
\textsc{RandomEv} & 0.50 $\pm$ 0.06 & 0.74 $\pm$ 0.16 & 0.81 $\pm$ 0.15 & 0.79 $\pm$ 0.15 & 8.55 $\pm$ 2.43 & 5.50 $\pm$ 2.63 & 5.60 $\pm$ 2.45 & 9.80 $\pm$ 2.21 & 7.36 $\pm$ 1.66 \\
\textsc{AME} & 0.50 $\pm$ 0.05 & 0.74 $\pm$ 0.16 & 0.81 $\pm$ 0.15 & 0.79 $\pm$ 0.15 & 8.54 $\pm$ 2.06 & 5.66 $\pm$ 2.59 & 5.58 $\pm$ 2.50 & 9.84 $\pm$ 2.17 & 7.41 $\pm$ 1.47 \\
$P_{y|x}$ & 0.52 $\pm$ 0.14 & 0.69 $\pm$ 0.16 & 0.73 $\pm$ 0.16 & \bf{0.70 $\pm$ 0.13} & 8.69 $\pm$ 3.10 & 9.66 $\pm$ 3.07 & 10.64 $\pm$ 2.46 & 3.73 $\pm$ 2.67 & 8.18 $\pm$ 2.13 \\
\textsc{DataOob} & 0.51 $\pm$ 0.15 & 0.71 $\pm$ 0.15 & 0.75 $\pm$ 0.16 & 0.73 $\pm$ 0.14 & 8.56 $\pm$ 3.38 & 8.55 $\pm$ 3.15 & 9.71 $\pm$ 2.56 & 6.04 $\pm$ 3.71 & 8.22 $\pm$ 2.44 \\
\textsc{Inf} & 0.46 $\pm$ 0.14 & 0.71 $\pm$ 0.16 & 0.79 $\pm$ 0.17 & 0.78 $\pm$ 0.14 & 9.65 $\pm$ 2.61 & 7.43 $\pm$ 3.17 & 7.08 $\pm$ 2.95 & 9.56 $\pm$ 2.79 & 8.43 $\pm$ 2.26 \\
$P_x$ & 0.53 $\pm$ 0.11 & 0.63 $\pm$ 0.15 & 0.69 $\pm$ 0.16 & 0.74 $\pm$ 0.14 & 7.59 $\pm$ 3.59 & 11.28 $\pm$ 2.68 & 10.38 $\pm$ 3.32 & 6.06 $\pm$ 3.27 & 8.83 $\pm$ 2.51 \\
\textsc{DataBanzhaf} & 0.45 $\pm$ 0.14 & 0.71 $\pm$ 0.16 & 0.78 $\pm$ 0.17 & 0.79 $\pm$ 0.14 & 9.92 $\pm$ 2.89 & 8.26 $\pm$ 3.24 & 7.49 $\pm$ 3.24 & 9.63 $\pm$ 2.72 & 8.83 $\pm$ 2.39 \\
\bottomrule      
\end{tabular}
}
\end{table*}

\end{document}